\documentclass{article} 
\usepackage{style,times}
\usepackage{booktabs}

\usepackage{amsmath,amsfonts,bm}

\def\eqref#1{equation~\ref{#1}}

\def\1{\bm{1}}

\DeclareMathAlphabet{\mathsfit}{\encodingdefault}{\sfdefault}{m}{sl}
\SetMathAlphabet{\mathsfit}{bold}{\encodingdefault}{\sfdefault}{bx}{n}

\usepackage{wrapfig}
\usepackage{hyperref}
\usepackage{url}
\usepackage{algorithm}
\usepackage{algorithmic} 
\usepackage{amsmath}     
\usepackage{amsfonts}    
\usepackage{amssymb}
\usepackage{amsthm}
\usepackage{multirow}
\theoremstyle{plain}  \newtheorem{proposition}{Proposition}
\theoremstyle{remark} 
\usepackage{graphicx}
\usepackage{enumitem}
\usepackage{xcolor}
\usepackage{authblk}

\usepackage{tcolorbox}
\tcbuselibrary{skins, breakable}

\newtcolorbox{promptbox}[1]{
  colback=gray!5, colframe=black!60, boxrule=0.5pt, arc=1pt,
  breakable, fonttitle=\bfseries\small, title={#1},
  left=6pt, right=6pt, top=4pt, bottom=4pt
}
\newtcolorbox{modbox}[1]{
  colback=blue!4, colframe=blue!45!black, boxrule=0.4pt, arc=1pt,
  breakable, fonttitle=\bfseries\footnotesize, title={#1},
  left=5pt, right=5pt, top=3pt, bottom=3pt
}
\newcommand{\slot}[1]{\textcolor{blue!45!black}{\textbf{\{\,#1\,\}}}}

\vspace{-10pt}
\title{Beyond Prediction: A Component-Level Study of Concept Recovery in In-Context Learning}

\author{
    Xufeng Liu \quad Fang Wan \quad Yi Liu\thanks{Correspondence to yi.liu.4@stonybrook.edu} \\
    \small  Stony Brook University
}

\finalcopy 
\begin{document}

\maketitle
\vspace{-10pt}
\begin{abstract}
In-context learning (ICL) enables language models to make accurate predictions from demonstrations, but prediction success alone does not reveal the extent to which the underlying concept has been recovered. Existing evaluations typically assess recovery at the level of the whole concept, making it difficult to identify which components are recovered and to separate model limitations from ambiguity in the demonstrations. We introduce a \textbf{component-level framework for concept recovery} that decomposes concepts into their defining components and uses controlled structural interventions to evaluate both component-level and whole-concept recovery under identifiable demonstrations. Across four instruction-tuned LLMs and three composition forms, we find that models often achieve high prediction accuracy despite poor whole-concept recovery, with systematic over-specification of relevant attributes and level constraints. Individual structural components are generally insufficient for whole-concept recovery, whereas jointly specifying complementary components, particularly relevant attributes and their level constraints, substantially improves recovery. However, this improvement is fragile: a single plausible but unnecessary level constraint can substantially disrupt recovery. Finally, joint supervision of recovery and implicit prediction substantially improves recovery while preserving prediction performance, with the learned capability transferring across composition forms. Overall, our component-level analysis provides a more diagnostic view of what models recover from demonstrations than the existing holistic view.
\end{abstract}

\section{Introduction}
\vspace{-10pt}
Large language models (LLMs) have achieved strong performance across a wide range of tasks, from academic reasoning~\citep{guo2025deepseek,singh2025openai} to coding~\citep{el2025competitive,comanici2025gemini} and scientific applications~\citep{boiko2023autonomous,gottweis2026accelerating}. Yet, successful behavior alone does not necessarily reveal whether a model has recovered the underlying structure that supports its predictions. Understanding what models infer from task-specific evidence therefore remains an important challenge for LLM interpretability~\citep{luo2024understanding,singh2024rethinking,huben2024sparse,chen2025reasoning}.

In-context learning (ICL) provides a natural setting to study this question. By conditioning on a small set of labeled demonstrations, LLMs can adapt their predictions to the task specified by the demonstrations on subsequent queries~\citep{brown2020language,min2022rethinking}. A natural interpretation is that the demonstrations provide information about an underlying task concept that can be inferred and applied to new queries. This raises a more direct question: \emph{when ICL succeeds, to what extent does the model recover the underlying concept from the demonstrations?} In particular, accurate query prediction does not necessarily imply faithful recovery of the concept that generates the labels. We study this prediction--recovery gap in controlled ICL settings where demonstrations are generated from predefined target concepts, allowing recovery to be evaluated against explicit ground truth.

Several lines of work have investigated the mechanisms underlying successful ICL~\citep{olsson2022context,dai2023can,pan2023context,agarwal2024many}. Related approaches to rule or hypothesis recovery, including MIRAGE~\citep{li2025mirage} and iterative hypothesis refinement~\citep{qiu2024phenomenal}, further examine whether models can explicitly infer rules from demonstrations. However, recovery in these settings is typically evaluated at the level of the recovered rule as a whole. Such holistic evaluation can determine whether a recovered concept is correct, but provides limited information about \emph{which concept-defining components are recovered, missed, or incorrectly included}. It also does not by itself establish whether the demonstrations uniquely identify the target concept, making it difficult to distinguish limitations in model recovery from ambiguity in the evidence. These limitations leave open a more diagnostic question: \emph{what does a model recover when it fails to recover the concept as a whole?}

\begin{wrapfigure}{r}{0.6\textwidth}
\vspace{-10pt}
\centering
\includegraphics[width=0.6\textwidth]{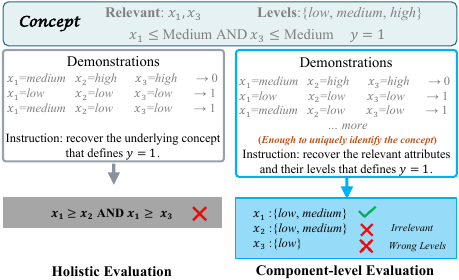}
\vspace{-10pt}
\caption{\textbf{Holistic versus component-level evaluation of concept recovery.} Given a target concept, we generate demonstrations from the concept and provide them to an LLM for concept recovery. \emph{Holistic evaluation} assesses whether the recovered concept is correct as a whole, without revealing which concept-defining components are recovered correctly or incorrectly. We instead adopt a \emph{component-level view}, using demonstrations sufficient to uniquely identify the target concept and evaluating recovery at the level of individual concept-defining components.}
\label{fig:motivation}
\vspace{-10pt}
\end{wrapfigure}
To address this question, we move beyond holistic evaluation to a \textbf{component-level framework for concept recovery}, as illustrated in Figure~\ref{fig:motivation}. We represent each target concept through its constituent components and evaluate recovery separately at the component and whole-concept levels. We first ensure that the demonstrations are sufficient to uniquely identify the target concept, separating model-side recovery errors from ambiguity in the evidence. We then use controlled interventions to test whether individual or jointly specified concept-defining components support whole-concept recovery, and whether this recovery remains robust when a plausible but unnecessary structural element is introduced. Finally, we use supervised training to test whether the same structured recovery capability can be learned while retaining implicit prediction.

Our experiments reveal several empirical patterns in concept recovery. (1) Models often achieve high query prediction accuracy despite poor whole-concept recovery. Component-level analysis further reveals systematic over-specification: models tend to include irrelevant attributes and overly broad level constraints. (2) Individual structural components are generally insufficient for whole-concept recovery, whereas jointly specifying complementary concept-defining components, particularly the relevant attributes and their level constraints, substantially improves recovery. However, recovery remains fragile: a single plausible but unnecessary level constraint can substantially disrupt it. (3) Structured concept recovery is learnable. Inference-time reasoning alone does not reliably improve recovery, whereas joint training with recovery and implicit-prediction supervision substantially improves recovery while preserving implicit-prediction performance, and the learned capability transfers across composition forms. 

We summarize our contributions as follows:
\begin{itemize}
\item \textbf{A component-level characterization of concept recovery.} We decompose concept recovery into constituent structural components and evaluate which components are correctly recovered, missed, or incorrectly included, going beyond a holistic assessment of whether the recovered concept is correct.
\item \textbf{A controlled analysis of structural information for recovery.} We use identifiable demonstrations and controlled interventions to characterize how individual structural components and their joint specification support whole-concept recovery, and how plausible but unnecessary information can interfere with recovery.
\item \textbf{A learnable concept-recovery capability.} We show that structured concept recovery can be learned through supervised training, improving recovery while preserving implicit prediction and transferring across composition forms.
\end{itemize}
\vspace{-5pt}

\section{Background and Related Work}\label{sec:background_related_work}
\vspace{-5pt}
\subsection{Problem Formulation}\label{sec:icl_background}
\vspace{-5pt}

We consider an in-context learning task governed by an underlying concept $c$, which specifies the input–output relationship that defines the task. We use ``concept’’ as an abstract term for this latent, task-defining object, which may take different forms depending on the task, such as a rule, function, program, or other structured specification. Given a concept $c$, demonstrations are generated by applying $c$ to a set of inputs,
\[
D=\{(x^{(i)},y^{(i)})\}_{i=1}^{H},
\qquad y^{(i)}=c(x^{(i)}).
\]
A query $x^q$ from the same task is likewise associated with $y^q=c(x^q)$. Given the demonstrations and a query, an ICL model can directly produce the query output
$
\hat y^q=\mathrm{LLM}_{\mathrm{Im}}(D,x^q).
$
We refer to this as \textbf{implicit prediction}: the model produces the query label without being explicitly required to state the underlying concept.

We can instead evaluate whether the model can \textbf{recover the underlying concept} itself from the demonstrations,
$
\hat c=\mathrm{LLM}_{\mathrm{Rec}}(D).
$
Thus, prediction uses the demonstrations to answer a new query, whereas concept recovery uses the demonstrations to infer the latent object $c$ that defines the task. This distinction allows us to ask our central question: \emph{when implicit prediction succeeds, to what extent does the model recover the underlying concept?}

\subsection{Related Work}
\vspace{-5pt}
Prior work has studied how demonstrations convey task information from several perspectives. Some work examines how models infer hidden tasks or task structure from demonstrations~\citep{xie2021explanation,hahn2023theory,pan2023context}. Other work studies how properties of demonstrations affect successful prediction, including labels, input distributions, format, similarity, ordering, and retrieval~\citep{zhao2021calibrate,min2022rethinking,liu2022makes,lu2022fantastically,rubin2022learning}. Related work also considers more explicit forms of task knowledge, including natural-language task descriptions, instructions, and rules~\citep{honovich2023instruction,zhu2023large}. Another line of work studies how task- or function-level information is induced and represented in ICL~\citep{olsson2022context,akata2026context}. These works motivate viewing ICL as more than direct label prediction, but do not directly address the structured recovery of the underlying task-defining concept considered here. A more detailed discussion of related work is provided in Appendix~\ref{appendix:related_work}.

More directly related to concept recovery are works that explicitly ask models to infer an underlying rule or hypothesis from observations. \citet{qiu2024phenomenal} study inductive reasoning through iterative hypothesis refinement, where models generate and refine textual hypotheses before applying them to new instances. MIRAGE~\citep{li2025mirage} explicitly studies the distinction between rule induction and example inference, showing that models can predict unseen examples correctly without necessarily recovering the correct generalized rule. These studies highlight the distinction between successful prediction and recovery of the underlying rule.

However, these lines of work do not directly address two aspects that are central to our setting. First, they typically evaluate the recovered object as a whole, providing limited information about which concept-defining components are correctly recovered, missed, or incorrectly included, particularly when demonstrations contain plausible but irrelevant structure. Second, they do not explicitly control whether the demonstrations uniquely identify the target concept. When multiple concepts are consistent with the same demonstrations, failure to recover $c$ may reflect ambiguity in the evidence rather than a limitation of the model’s recovery capability. We therefore ask a more specific evaluation question: \emph{when $D$ uniquely identifies $c$, which components of the target concept are recovered, missed, or incorrectly included?}

\section{Component-Level Concept Recovery}\label{sec:framework}
\vspace{-5pt}
Our framework addresses the prediction–recovery gap through a \textbf{structured view of concept recovery}, as illustrated in Figure~\ref{fig:framework}. First, we represent the underlying concept through its constituent components and evaluate recovery at the component level (Sec.\ref{sec:measurement}), establishing \emph{what} is being recovered and how recovery errors can be localized. Second, building on this representation, we use controlled interventions on concept-defining information and plausible but unnecessary structural information to characterize which combinations of structural information support recovery (Sec.\ref{sec:localizing}). Finally, we use supervised learning to explicitly teach the same structured recovery capability while retaining the original implicit-prediction objective (Sec.\ref{sec:repair}). These stages move from \emph{measuring} concept recovery, to \emph{localizing} its structural dependencies, to \emph{learning} the recovery capability itself.

\begin{figure}[H]
    \centering
    \includegraphics[width=\linewidth]{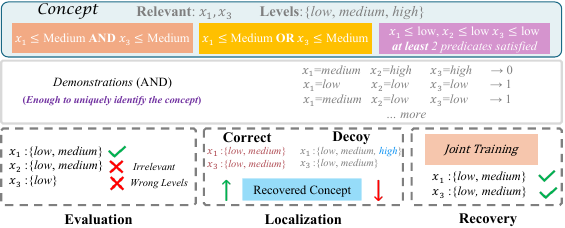}
    \vspace{-10pt}
    \caption{\textbf{Overview of our component-level framework for concept recovery.} The top panel illustrates the three composition forms, while the lower panels show the evaluation, localization, and recovery stages. We first decompose recovered concepts into defining components and evaluate at the component level. We then use controlled interventions on structural information to localize their effects on recovery. Finally, we use supervised learning to finetune the model to recover the structured concept while preserving its implicit-prediction capability. }
    \label{fig:framework}
    \vspace{-5pt}
\end{figure}

\subsection{Structured Concept Representation}\label{sec:measurement}
\vspace{-5pt}

\textbf{From Holistic to Structured Representation.} We represent each target concept as a structured specification composed of concept-defining components. This representation makes the underlying concept explicit and provides a common basis for evaluating recovery at both the component and concept levels.

We instantiate our framework in a STAGGER-style controlled concept-learning environment~\citep{schlimmer1986incremental}, which provides a suitable setting for our framework because concepts are defined through multiple interacting attributes rather than represented only as an input–output mapping. This structure allows us to study whether a model recovers \emph{which attributes matter}, \emph{what constraints they impose}, and \emph{how these constraints are combined}, while retaining explicit ground truth for each component.

Each concept is generated from $K$ predicate slots and a known compositional form $g$, of which only a subset may impose effective constraints. Each effective predicate is defined by an input attribute together with a constraint on its possible levels. For an attribute $j$, we represent its constraint by a set of satisfying levels $L_j\subseteq\mathcal{V}$, where $\mathcal{V}=\{\textit{low},\textit{medium},\textit{high}\}$. The compositional form $g$ specifies how these predicates are combined: conjunction (\textsc{AND}) requires all predicates to be satisfied, disjunction (\textsc{OR}) requires at least one, and $m$-of-$K$ requires at least $m$ predicates to be satisfied. Thus, the concept-specific instantiation can be represented as
\begin{equation}
c_g=
\begin{cases}
(A,\{L_j\}_{j\in A}), & g\in\{\mathrm{AND},\mathrm{OR}\},\\
(A,\{L_j\}_{j\in A},m), & g=\mathrm{m\text{-}of\text{-}K},
\end{cases}
\label{eq:task_spec}
\end{equation}
where $A\subseteq\{1,\ldots,M\}$ denotes the indices of the effective relevant attributes, $\{L_j\}_{j\in A}$ specifies their level constraints, and $m$ specifies the threshold when applicable. We use $K$ to denote the nominal number of predicate slots; some predicate slots may impose no effective constraint, so $|A|\leq K$. Although $K$ can provide structural information for recovery, it is not included in $c_g$ because it specifies the nominal predicate capacity rather than an effective constraint on the input–output concept. The composition form $g$ is known and is not itself a recovery target. Concept recovery therefore targets the concept-defining specification $c_g$.

\textbf{Component-Level Concept Recovery.} The target concept determines the output of each demonstration through its concept-defining components, while each observation is represented over the full candidate attribute space. Thus, only a subset of the observed attributes belongs to the target concept, whereas the remaining attributes are irrelevant to the output. Across demonstrations, the correspondence between attribute values and outputs provides evidence about which attributes are relevant and what value constraints they impose. Details of concept construction and demonstration generation are provided in Appendix~\ref{app:task_construction}.

Given the ground-truth concept, we evaluate recovery at the level of its constituent components: the relevant attribute set $A$, the level constraint $L_j$ for each $j\in A$, and, for $m$-of-$K$ concepts, the threshold $m$. This allows us to distinguish components that are correctly recovered from those that are missed or incorrectly specified. This component-level measurement goes beyond whole-concept correctness, allowing us to characterize \emph{what} models recover and \emph{which} parts of the underlying concept remain unresolved.

\textbf{Identifiability of the Target Concept.} To separate recovery errors from ambiguity in the demonstrations, we require the demonstration set $D$ to uniquely identify the target concept within the admissible hypothesis class. Following the version-space formulation, $D$ is concept-identifying when the version space contains no behaviorally distinct hypothesis other than the target concept. We verify this condition at the deployed demonstration budgets by exhaustive enumeration and theoretical analysis; details are provided in Appendix~\ref{app:ident}.

\subsection{Localizing Information Dependencies in Concept Recovery}\label{sec:localizing}
\vspace{-5pt}
Component-level evaluation localizes recovery errors but does not determine whether explicitly providing the corresponding structural information improves recovery. We therefore use controlled information interventions to isolate how different structural components support concept recovery.

In each intervention, we fix the target concept and its demonstrations and vary only the ground-truth structural information explicitly provided to the recovery model. Let $S$ denote the set of structural information components provided in addition to the demonstrations. In our interventions, $S$ may specify $K$, $A$, $L$, $m$, or combinations thereof, including a plausible decoy constraint $\tilde L$. The recovery model produces $\hat c_S^{(i)}=
\mathrm{LLM}_{\mathrm{Rec}}\!\left(D^{(i)},S^{(i)}\right),$ with $c^{*(i)}$ denotes the corresponding ground-truth target concept for instance $i$. This formulation allows us to compare recovery under different information states while keeping the underlying task instance and recovery procedure fixed.

\textbf{Whole-Concept Recovery.} We evaluate whether the recovered specification is behaviorally equivalent to the target concept over the full input space:
\begin{equation}
\hat c_S^{(i)} \equiv c^{*(i)}
\quad\Longleftrightarrow\quad
\forall x\in\mathcal{X}_i,\;
\hat c_S^{(i)}(x)=c^{*(i)}(x),
\label{eq:concept_equivalence}
\end{equation}
where $\mathcal{X}_i$ denotes the input space of the $i$-th target concept. We evaluate this equivalence exhaustively when the input space is tractable; otherwise, we use a fixed deterministic subsample of at most $4096$ contexts.

The whole-concept recovery rate under structural information $S$ is defined as
\begin{equation}
R(S)=
\frac{1}{N}
\sum_{i=1}^{N}
\mathbb{I}\!\left[
\hat c_S^{(i)}\equiv c^{*(i)}
\right],
\label{eq:concept_recovery}
\end{equation}
where $N$ is the number of target concept instances. Thus, $R(S)$ measures whether the recovered specification defines the same input–output concept as the target, complementing component-level metrics that evaluate individual elements of the recovered specification.

\textbf{Decoy-Level Interference.} We next test the robustness of recovery to plausible but unnecessary structural information. For each relevant attribute $j\in A$, we augment its ground-truth level constraint with one additional level $\tilde{\ell}_j\notin L_j$, yielding $\tilde L_j=L_j\cup{\tilde{\ell}_j}$. We compare whole-concept recovery under the correct and augmented structural information, namely $R(A,L)$ and $R(A,L,\tilde L)$. The target concepts, demonstrations, recovery procedure, and evaluation remain fixed across conditions, so the comparison isolates the effect of adding plausible but unnecessary level information.

\subsection{Learning Structured Concept Recovery}\label{sec:repair}
\vspace{-5pt}
The preceding framework defines concept recovery as structured prediction over the concept representation in Eq.~(\ref{eq:task_spec}) and identifies the structural information that can support this recovery. We next ask whether this recovery capability can be explicitly learned. We therefore supervise the structured concept representation while retaining the original implicit-prediction objective.

\textbf{Joint Recovery and Prediction Supervision.} We treat concept recovery and implicit prediction as complementary objectives. Recovery supervision provides demonstrations and requires the model to explicitly state the underlying concept, including its attributes, level constraints, and, when applicable, threshold $m$. Implicit-prediction supervision retains the original ICL objective by providing demonstrations and a query and supervising the corresponding output. For each training concept, we construct one recovery example and two balanced prediction examples, one for each label, and jointly optimize them with the standard token-level cross-entropy objective. This objective directly supervises the structured recovery target while preserving the original implicit-prediction objective.

\textbf{Training and Transfer.} We train on $m$-of-$K$ concepts because this family exercises the full structured specification $(A,L,m)$. We evaluate the trained model on held-out concepts within the training family and on concepts from unseen \textsc{AND} and \textsc{OR} composition forms. This evaluates whether the learned recovery capability transfers beyond the composition form observed during training.

\section{Experimental Studies} \label{sec:exp}
\vspace{-5pt}

We first introduce the experimental setup, including the target concepts, evaluated models, and evaluation metrics (Sec.\ref{sec:exp_setup}). We then examine concept recovery at the component level and characterize systematic recovery errors (Sec.\ref{exp:measure}). Next, we use controlled interventions to test whether constraining individual components or jointly specifying complementary components enables whole-concept recovery, and whether successful recovery is robust to plausible but unnecessary structural information (Sec.\ref{exp:localize}). Finally, we evaluate whether structured concept recovery can be learned through supervised training while retaining implicit prediction (Sec.\ref{exp:repair}).

\subsection{Experimental Setup} \label{sec:exp_setup}
\vspace{-5pt}

\textbf{Concepts and Data.} We instantiate $48$ target concepts from the STAGGER-style environment, covering three composition forms (\textsc{AND}, \textsc{OR}, and $m$-of-$K$), $2$ nominal arities ($K\in\{3,4\}$), and $8$ generator seeds. For each concept, we construct balanced demonstration sets with $H\in\{100,200\}$ labeled examples and a balanced held-out query set $Q$ containing $32$ probes per concept and evaluation condition.

\textbf{Models.} We evaluate four instruction-tuned LLMs: Qwen2.5-72B-Instruct, Llama-3.1-70B-Instruct, Gemma-2-27B-IT, and Mistral-Small-24B-Instruct-2501. All models use publicly released checkpoints and are evaluated under the same prompting protocol with greedy decoding ($T=0$).

\textbf{Evaluation Metrics.} We evaluate two aspects: \emph{concept application} and \emph{concept recovery}. \textbf{Concept Application} measures prediction accuracy on the balanced held-out query set $Q$. We denote the accuracy of \emph{implicit prediction} by $Q_{\mathrm{imp}}$ and that of \emph{explicit prediction}—applying the recovered concept to the same queries—by $Q_{\mathrm{exp}}$. \textbf{Concept Recovery} measures whether the recovered concept matches the target concept. We use the whole-concept recovery rate $R(S)$ defined in Eq.~(\ref{eq:concept_recovery}), and $\mathrm{Acc}_A$, $\mathrm{Acc}_L$, and $\mathrm{Acc}_m$ for component-level recovery of relevant attributes, level constraints, and the threshold, respectively. $Q_{\mathrm{exp}}$ measures how accurately the recovered concept is applied on held-out queries, whereas $R(S)$ measures whole-concept recovery over the input space and is therefore a stricter measure of whether the underlying concept has been recovered.

\textbf{Experimental Protocol.} All experiments use the same modular prompt templates, with only the relevant prompt components varied across experimental conditions. The prompt templates and inference-time variants are provided in Appendix~\ref{app:prompt}.

\subsection{Component-Level Evaluation of Concept Recovery}\label{exp:measure}
\vspace{-5pt}

We first examine whether accurate implicit prediction is accompanied by faithful explicit concept recovery. We compare the model’s implicit predictions with predictions obtained by executing its recovered concept on the same held-out queries. As shown in Table~\ref{result:implicit_vs_explicit}, implicit prediction achieves $0.67$–$0.83$ accuracy, whereas the recovered concepts achieve only $0.50$–$0.63$. Thus, models can make substantially more accurate predictions than the concepts they explicitly recover from the same demonstrations.

We next ask what underlies this prediction–recovery gap. Our component-level evaluation shows that the errors are systematic rather than random. Exact recovery of the relevant-attribute set $A$ is uniformly $0.00$, while models predict $6.56$–$7.50$ attributes as relevant on average. Since each input contains $M=K+4\in\{7,8\}$ attributes, with $K\in\{3,4\}$ denoting the nominal number of predicate slots, models tend to include nearly the entire candidate attribute space rather than isolate the attributes that define the target concept. Among correctly identified relevant attributes, level-constraint mismatches are also biased toward \emph{Over} errors, while threshold errors in $m$-of-$K$ concepts are predominantly \emph{Too High} (Appendix~\ref{app:level_error}; Appendix~\ref{app:threshold_error}). Together, these results indicate a consistent tendency to \textbf{over-specify the target concept}: models tend to include additional attributes, satisfying levels, or required predicates.

\vspace{-20pt}
\begin{table}[ht]
\caption{Concept application and recovery under demonstrations alone. $Q_{\mathrm{imp}}$ and $Q_{\mathrm{exp}}$ denote prediction accuracy on the held-out query set $Q$ using implicit prediction and the recovered concept, respectively. $\mathrm{Acc}_A$, $\mathrm{Acc}_{L|A}$, and $\mathrm{Acc}_m$ measure recovery of relevant attributes, level constraints conditional on correctly identified attributes, and the composition threshold, respectively; $|\hat A|$ reports the number of predicted relevant attributes.}
\vspace{2pt}
\label{result:implicit_vs_explicit}
\centering
\small
\begin{tabular}{llcc|cccc}
\toprule
& & \multicolumn{2}{c|}{Concept Application} & \multicolumn{4}{c}{Concept Recovery} \\
Form & Model
& $Q_{\mathrm{imp}}$
& $Q_{\mathrm{exp}}$
& $\mathrm{Acc}_A$ 
& $|\hat A|$
& $\mathrm{Acc}_{L|A}$ 
& $\mathrm{Acc}_m$  \\
\midrule
\multirow{4}{*}{\textsc{AND}}
 & Qwen2.5-72B    & 0.83 & 0.63 & 0.00 & 7.16 & 0.70 & -- \\
 & Llama-3.1-70B  & 0.73 & 0.56 & 0.00 & 7.31 & 0.71 & -- \\
 & Gemma-2-27B    & 0.72 & 0.58 & 0.00 & 7.47 & 0.38 & -- \\
 & Mistral-24B    & 0.77 & 0.52 & 0.00 & 7.47 & 0.62 & -- \\
\midrule
\multirow{4}{*}{\textsc{OR}}
 & Qwen2.5-72B    & 0.80 & 0.50 & 0.00 & 6.91 & 0.05 & -- \\
 & Llama-3.1-70B  & 0.74 & 0.51 & 0.00 & 7.50 & 0.15 & -- \\
 & Gemma-2-27B    & 0.77 & 0.50 & 0.00 & 7.06 & 0.05 & -- \\
 & Mistral-24B    & 0.69 & 0.51 & 0.00 & 7.47 & 0.24 & -- \\
\midrule
\multirow{4}{*}{$m$-of-$K$}
 & Qwen2.5-72B    & 0.73 & 0.53 & 0.00 & 6.56 & 0.39 & 0.00 \\
 & Llama-3.1-70B  & 0.71 & 0.51 & 0.00 & 7.06 & 0.29 & 0.03 \\
 & Gemma-2-27B    & 0.67 & 0.51 & 0.00 & 6.97 & 0.21 & 0.03 \\
 & Mistral-24B    & 0.68 & 0.51 & 0.00 & 7.47 & 0.32 & 0.03 \\
\bottomrule
\end{tabular}
\vspace{-10pt}
\end{table}

\begin{wraptable}{r}{0.6\textwidth} 
  \vspace{-20pt} 
\caption{Paired comparison of implicit prediction and predictions from the model’s stated rule. $b$ counts queries correct under implicit prediction but incorrect under the stated rule, while $c$ counts the reverse pattern. \emph{Cond. Loss} measures the fraction of implicitly correct queries overturned by the stated rule. McNemar’s exact test compares the discordant counts. $\downarrow$ denotes lower-is-better.}
\vspace{2pt}
\label{tab:within_item}
\centering
\small
\begin{tabular}{lcccc}
\toprule
Model & $b$ & $c$ & Cond. Loss $\downarrow$ & McNemar $p$ \\
\midrule
Qwen2.5-72B & 84 & 18 & 0.33 & $2.3{\times}10^{-11}$ \\
Llama-3.1-70B & 76 & 33 & 0.32 & $4.7{\times}10^{-5}$ \\
Gemma-2-27B & 86 & 34 & 0.39 & $2.3{\times}10^{-6}$ \\
Mistral-Small-24B & 90 & 22 & 0.39 & $6.0{\times}10^{-11}$ \\
\bottomrule
\end{tabular}
\vspace{-10pt}
\end{wraptable}
A paired comparison on the same held-out queries further tests this dissociation, avoiding differences due to independently sampled test sets. Across models, the recovered rule overturns $0.32$–$0.39$ of predictions that are correct under implicit prediction, while the reverse pattern is consistently less frequent ($c<b$). McNemar’s exact test rejects symmetry between the discordant prediction counts for all models ($p<10^{-4}$). This paired analysis therefore shows that the prediction–recovery gap is not driven simply by different query samples: the recovered rules systematically overturn a substantial fraction of predictions that the models otherwise get correct.

We then rule out general difficulty in applying an explicit concept: directly providing the target concept yields substantially higher held-out accuracy (Appendix~\ref{app:addtion_apply}). We also examine whether the component-level pattern extends beyond the controlled environment. On \textsc{MONK-1}, recovered concepts tend to include more attributes than the ground-truth relevant-attribute set, while analogous attribute-count patterns are observed on Mushroom (Appendix~\ref{app:other_dataset}).

\textbf{Takeaway.} Accurate implicit prediction does not imply faithful explicit concept recovery. \textbf{Component-level analysis reveals a systematic tendency to over-specify the target concept}, rather than merely failing to recover it at random. This pattern also persists beyond the controlled STAGGER-style environment.

\subsection{Structural Information Shapes Concept Recovery}\label{exp:localize}
\vspace{-5pt}
The component-level analysis reveals systematic over-specification in the recovered concepts. We therefore next ask whether providing the correct value of an individual structural component is sufficient to correct the corresponding error and recover the full concept. To test this, we keep the target concepts and demonstrations fixed while providing selected ground-truth structural components. Because the main error in $A$ is over-inclusion, we also provide $K$, the nominal number of predicate slots, to test whether constraining the size of the relevant attribute is sufficient for recovery.

Table~\ref{tab:staircase} compares recovery when individual structural components are provided separately or jointly. Providing $K$, $A$, $L$, or $m$ individually is generally insufficient for whole-concept recovery. The main exception is \textsc{AND}, where providing $A$ alone yields $0.19$–$0.84$ recovery. However, recovery remains near zero for \textsc{OR} and $m$-of-$K$. These results show that \textbf{identifying a component-level recovery error does not by itself provide sufficient information to recover the whole concept}. We next jointly provide $(A,L)$. This substantially improves recovery across the three concept families. Thus, \textbf{jointly specifying complementary defining components enables substantially more complete whole-concept recovery}, while $m$ provides the additional information required by $m$-of-$K$.

\begin{table}[t]
\vspace{-12pt}
\caption{Concept recovery under progressively revealed structural information. Each entry reports whole-concept recovery $R(S)$ after providing the indicated ground-truth information $S$ in addition to the demonstrations. $\varnothing$ denotes demonstrations alone; $K$, $A$, $L$, and $m$ denote the corresponding structural components.}
\vspace{2pt}
\label{tab:staircase}
\centering
\small
\begin{tabular}{llccccccc}
\toprule
Form & Model
& $R(\varnothing)$
& $R(K)$
& $R(A)$
& $R(L)$
& $R(m)$
& $R(A,L)$
& $R(A,L,m)$ \\
\midrule
\multirow{4}{*}{\textsc{AND}}
 & Qwen2.5-72B   & 0.06 & 0.09 & 0.81 & 0.00 & -- & 1.00 & -- \\
 & Llama-3.1-70B & 0.00 & 0.03 & 0.84 & 0.03 & -- & 1.00 & -- \\
 & Gemma-2-27B   & 0.00 & 0.00 & 0.19 & 0.06 & -- & 0.97 & -- \\
 & Mistral-24B   & 0.00 & 0.00 & 0.66 & 0.06 & -- & 1.00 & -- \\
\midrule
\multirow{4}{*}{\textsc{OR}}
 & Qwen2.5-72B   & 0.00 & 0.00 & 0.00 & 0.03 & -- & 0.94 & -- \\
 & Llama-3.1-70B & 0.00 & 0.00 & 0.00 & 0.09 & -- & 0.88 & -- \\
 & Gemma-2-27B   & 0.00 & 0.00 & 0.00 & 0.00 & -- & 0.81 & -- \\
 & Mistral-24B   & 0.00 & 0.00 & 0.00 & 0.00 & -- & 1.00 & -- \\
\midrule
\multirow{4}{*}{$m$-of-$K$}
 & Qwen2.5-72B   & 0.00 & 0.00 & 0.00 & 0.00 & 0.00 & 0.47 & 0.94 \\
 & Llama-3.1-70B & 0.00 & 0.00 & 0.00 & 0.00 & 0.00 & 0.94 & 1.00 \\
 & Gemma-2-27B   & 0.00 & 0.00 & 0.00 & 0.00 & 0.00 & 0.00 & 0.81 \\
 & Mistral-24B   & 0.00 & 0.00 & 0.00 & 0.00 & 0.00 & 1.00 & 1.00 \\
\bottomrule
\end{tabular}
\vspace{-10pt}
\end{table}

\begin{wraptable}{r}{0.48\textwidth}
  \vspace{-20pt} 
  \caption{Sensitivity of concept recovery to decoy-level information. We compare $R(A,L)$ with recovery after augmenting each level constraint with one plausible but unnecessary level, denoted by $\tilde L$.}
  \vspace{2pt}
  \label{tab:counterfactual_level}
  \centering
  \small
  \begin{tabular}{llcc}
\toprule
Form & Model & $R(A,L)$ & $R(A,\tilde{L})$  \\
\midrule
\multirow{4}{*}{\textsc{AND}}
 & Qwen2.5-72B    & 1.00 & 0.97  \\
 & Llama-3.1-70B  & 1.00 & 0.59  \\
 & Gemma-2-27B    & 0.97 & 0.31   \\
 & Mistral-24B    & 1.00 & 0.66  \\
\midrule
\multirow{4}{*}{\textsc{OR}}
 & Qwen2.5-72B    & 0.88 & 0.19  \\
 & Llama-3.1-70B  & 1.00 & 0.38  \\
 & Gemma-2-27B   & 0.81 & 0.00  \\
 & Mistral-24B    & 1.00 & 0.16\\
\midrule
\multirow{4}{*}{$m$-of-$K$}
 & Qwen2.5-72B    & 0.47 & 0.00   \\
 & Llama-3.1-70B  & 0.94 & 0.16  \\
 & Gemma-2-27B   & 0.00 & 0.00  \\
 & Mistral-24B    & 1.00 & 0.03  \\
\bottomrule
\end{tabular}
  \vspace{-10pt} 
\end{wraptable}
Having established that $(A,L)$ provides a condition under which recovery can be highly successful, we next test whether this recovery is robust to plausible but unnecessary structural information. We deliberately contaminate the successful specification by augmenting each ground-truth level constraint with one plausible but unnecessary level, while keeping $A$ fixed. This intervention is motivated by the systematic over-specification of level constraints observed in Sec.~\ref{exp:measure}. If the model can distinguish concept-defining information from unnecessary but plausible structure, the correct joint specification should remain sufficient despite this single decoy level.

As shown in Table~\ref{tab:counterfactual_level}, this single decoy-level intervention reduces whole-concept recovery in 11 of 12 model--form conditions. The degradation is especially pronounced for \textsc{OR} and $m$-of-$K$. These results show that \textbf{successful recovery remains sensitive to plausible but unnecessary structural information}, even when the correct joint specification is already provided.

\textbf{Takeaway.} Individual component constraints are generally insufficient for whole-concept recovery, whereas \textbf{jointly specifying complementary defining components enables substantially more complete recovery}. However, this successful recovery can be disrupted by a single plausible but unnecessary structural element.

\subsection{Learning Structured Concept Recovery}\label{exp:repair}
\vspace{-5pt}

\begin{wraptable}{r}{0.46\textwidth}
\vspace{-20pt}
\caption{Learning structured concept recovery on $m$-of-$K$ concepts. $Q_{\mathrm{imp}}$ denotes implicit prediction accuracy on held-out queries, while $R(\varnothing)$ measures whole-concept recovery from demonstrations alone. \emph{Recovery only} uses only structured recovery supervision, whereas \emph{Joint} combines recovery and implicit-prediction supervision.}
\vspace{2pt}
\label{tab:repair}
\centering
\small
\begin{tabular}{llcc}
\toprule
Form & Training
& $Q_{\mathrm{imp}}$
& $R(\varnothing)$ \\
\midrule
\multirow{3}{*}{$m$-of-$K$}
& No training      & 0.64 & 0.00 \\
& Recovery only    & 0.50 & 0.96 \\
& Joint            & 0.98 & 0.97 \\
\bottomrule
\end{tabular}
\vspace{-15pt}
\end{wraptable}

Before turning to supervised training, we first ask whether the prediction–recovery gap can be reduced through inference-time reasoning. As shown in Appendix~\ref{app:cot_selfrefine}, neither chain-of-thought reasoning~\citep{wei2022chain} nor iterative self-refinement~\citep{madaan2023self} consistently improves structured recovery, even when the identities of the relevant attributes are provided. We therefore ask whether this structured recovery capability can instead be learned directly through supervision. For supervised training, we use Qwen2.5-7B-Instruct with LoRA, intentionally using a smaller model than the 72B-scale models used for evaluation. Training details are provided in Appendix~\ref{app:training_details}.

We train on $m$-of-$K$ concepts, where recovery requires the full structured specification $(A,L,m)$. We compare three conditions: \emph{No training}, \emph{Recovery only}, which supervises only the structured concept representation, and \emph{Joint}, which combines recovery supervision with the original implicit-prediction objective. For each training concept, we construct one recovery example and two balanced prediction examples, one for each label, and jointly optimize the corresponding token-level cross-entropy objectives.

\begin{wraptable}{r}{0.46\textwidth}
\vspace{-20pt} 
\caption{Cross-family transfer after training on $m$-of-$K$ concepts. $Q_{\mathrm{imp}}$ denotes implicit prediction accuracy on held-out queries, while $R(\varnothing)$ measures whole-concept recovery from demonstrations alone.}
\vspace{2pt}
\label{tab:transfer}
\centering
\small
\begin{tabular}{llcc}
\toprule
Form & Training & $Q_{\mathrm{imp}}$ & $R(\varnothing)$\\
\midrule
\multirow{3}{*}{\textsc{and}}
& No training                        & 0.72& 0.00 \\
& Recovery only  & 0.50 & 0.95 \\
& Joint          & 0.97 & 0.94 \\
\midrule
\multirow{3}{*}{\textsc{or}}
 & No training                         & 0.66 & 0.00 \\
& Recovery only & 0.51 & 0.96 \\
 & Joint        & 0.99 & 0.98 \\
\bottomrule
\end{tabular}
\vspace{-15pt} 
\end{wraptable}

As shown in Table~\ref{tab:repair}, recovery-only training increases $R(\varnothing)$ from $0.00$ to $0.96$, but reduces $Q_{\mathrm{imp}}$ from $0.64$ to chance level. In contrast, joint training achieves both high recovery and high implicit-prediction accuracy, reaching $R(\varnothing)=0.97$ and $Q_{\mathrm{imp}}=0.98$. Thus, structured recovery can be learned without sacrificing implicit prediction when the two objectives are jointly supervised.

We further evaluate the performance on composition forms not observed during training. Table~\ref{tab:transfer} shows that recovery-only training retains high recovery on unseen \textsc{AND} and \textsc{OR} concepts, while joint training additionally preserves high implicit-prediction accuracy. 

\textbf{Takeaway.} Supervised recovery training substantially improves whole-concept recovery, but recovery-only supervision disrupts implicit prediction. \textbf{Joint recovery and prediction supervision learns structured recovery while preserving prediction capability}, and the learned recovery capability transfers to unseen composition forms.

\section{Conclusions}
\vspace{-5pt}
We introduced a component-level framework for concept recovery in in-context learning, moving beyond holistic evaluation by decomposing concepts into defining components and controlling the structural information available for recovery. Our findings show that successful prediction can coexist with poor concept recovery, and that component-level analysis provides a more diagnostic view of what models recover from demonstrations. Although our controlled concept-learning environment enables precise measurement and intervention, it leaves open whether the same patterns hold for more complex and open-ended concepts. Extending component-level recovery to natural-language rules, programs, and other structured task representations, and relating explicit recovery to the internal computations used for prediction, are promising directions for future work.

\newpage

\section*{AI use statement}

In this work, we used generative AI tools for improving the language and grammar of the manuscript. Additionally, we used generative AI tools for conducting our experiments, as LLMs are the primary subjects of our work. The reported experimental results are the direct outputs of interactions with these models, using structured prompts designed and refined solely by the authors. All research questions, methodologies, and experimental designs were developed by the authors. We have reviewed all AI-assisted work before submission, and we take responsibility for the final content of this work, including text, claims, or artifacts produced with the aid of generative AI.

\bibliography{reference,airs}
\bibliographystyle{style}

\clearpage
\appendix

\section{Detailed Comparison with Related Work}\label{appendix:related_work}
\vspace{-5pt}
This appendix provides a more detailed discussion of the related work summarized in Sec.~\ref{sec:background_related_work}, covering task and structure induction, behavioral studies of ICL, explicit task, rule, and function induction, task and function representations, and explicit concept recovery.

\subsection{Task and Structure Induction}\label{sec:related_work_task_structure}
\vspace{-5pt}
A line of work views demonstrations as evidence about an underlying task or latent structure. \citet{xie2021explanation} interpret ICL as implicit Bayesian inference over a latent task concept shared by the demonstrations, while \citet{hahn2023theory} study ICL as implicit structure induction and characterize how compositional structure can be inferred from examples. \citet{wang2023large} similarly model LLMs as latent-variable models and find empirical evidence that demonstrations induce latent task information. \citet{pan2023context} further distinguish task recognition from task learning, showing that models can recognize a task from demonstrations, even without ground-truth labels, while separately learning input--label mappings not available from pretraining.

These studies motivate viewing demonstrations as evidence about a latent task rather than merely as examples for direct prediction. Our work adopts this perspective but makes the recovery target explicit: instead of asking whether demonstrations induce latent task information that supports prediction, we ask which components of a predefined task-defining concept are recovered from demonstrations. This enables us to evaluate recovery at the component level and to verify whether the demonstrations uniquely identify the target concept.

\subsection{Behavioral Studies of ICL}\label{sec:related_work_behavior}
\vspace{-5pt}

A large body of work studies ICL from a behavioral perspective, examining which properties of demonstrations affect prediction. \citet{min2022rethinking} show that labels, input distributions, and sequence format contribute differently to ICL performance, while \citet{zhao2021calibrate} demonstrate that label distributions can substantially bias model predictions. Other work studies how the choice and ordering of demonstrations affect performance, including semantic similarity-based example selection, prompt ordering, and learned retrieval of useful examples~\citep{liu2022makes,lu2022fantastically,rubin2022learning}. Recent surveys organize these findings across different types of demonstration effects~\citep{dong2024survey}.

These studies show that the information contained in demonstrations can substantially affect prediction. Our work instead focuses on a different evaluation target: rather than measuring how demonstration properties affect query accuracy, we control the structural information provided for recovery to determine which parts of the underlying concept are recovered.

\subsection{Explicit Task, Rule, and Function Induction}\label{sec:related_work_explicit_induction}
\vspace{-5pt}
Another line of work explicitly studies whether models can infer task-level descriptions, rules, or functions from demonstrations. Instruction Induction~\citep{honovich2023instruction} prompts language models to generate natural-language instructions that describe the task underlying a set of input--output examples, and evaluates the induced instructions by their execution on new inputs. Hypotheses-to-Theories (HtT)~\citep{zhu2023large} similarly induces and verifies candidate rules from examples, building a rule library that can be used for subsequent reasoning. These works show that models can make task-level structure explicit rather than relying only on implicit prediction.

Related work studies the same question in more controlled function-learning settings. \citet{garg2022can} show that Transformers trained on synthetic data can learn unseen functions from in-context examples, including linear functions, sparse linear functions, neural networks, and decision trees. \citet{li2024language} further examine competing explanations of ICL and provide evidence that models can compose previously learned task components to implement novel tasks in context. More recent work studies function learning and function induction under controlled function classes and through internal mechanisms~\citep{akata2026context,ye2026function}.

Our setting shares this focus on an underlying structured object, but differs in the evaluation target: rather than evaluating induced descriptions through downstream execution, generalization, or task performance, we evaluate recovery against a target concept and decompose that concept into identifiable components. This allows us to distinguish components that are correctly recovered, missed, or incorrectly included.

\subsection{Task and Function Representations in ICL}\label{sec:related_work_representations}
\vspace{-5pt}
A complementary line of work studies how task- or function-level information induced by demonstrations is represented and used inside language models. Induction heads have been identified as an important mechanism for in-context pattern completion~\citep{olsson2022context}, while subsequent work shows that demonstrations can give rise to compact task- or function-level representations~\citep{hendel2023context,liu2023context,todd2024function}. More recent studies examine how structured task information is represented, including function induction, latent concept representations, and local task-specific representations~\citep{hong2026latent,zheng2026label}.

These studies provide evidence that task- or function-level information can be encoded in model representations and used to support prediction. However, representation-level evidence does not directly measure behavioral recovery of a predefined concept. In particular, the presence or decodability of task information in hidden states does not reveal which components of the target concept are recovered, whether irrelevant structure is also incorporated, or whether the available demonstrations uniquely identify the target. Our evaluation instead treats concept recovery as an explicit behavioral target and measures it at the level of individual concept-defining components.

\subsection{Explicit Concept Recovery and Prediction–Recovery Gaps}\label{sec:related_work_recovery}
\vspace{-5pt}
The most closely related work explicitly evaluates whether models recover the underlying rule or hypothesis from observed examples. \citet{qiu2024phenomenal} study inductive reasoning through iterative hypothesis refinement, where language models propose, select, and refine textual rules before applying them to new instances. By evaluating the intermediate hypotheses and their subsequent application, they identify a gap between rule induction and rule application. MIRAGE~\citep{li2025mirage} further separates rule induction from example inference in a controlled synthetic setting and shows that models can often infer unseen examples correctly without recovering the corresponding generalized rule. These results establish that successful prediction does not necessarily imply correct recovery of the underlying rule.

Existing evaluations primarily treat the recovered rule or hypothesis as a holistic object, making it difficult to determine which concept-defining components are correctly recovered, missed, or incorrectly included, particularly when observations contain plausible but irrelevant structure. Moreover, they do not explicitly verify whether the demonstrations uniquely identify the target concept, leaving recovery failure potentially confounded with ambiguity in the evidence. We instead construct demonstrations that uniquely identify a target concept and evaluate recovery at the level of explicit structural components. This formulation further enables controlled interventions on defining and irrelevant structural information, as well as supervised learning of recovery.

\section{Controlled Concept and Demonstration Generation}\label{app:task_construction}
\vspace{-5pt}
This section specifies how we instantiate the controlled concepts used in our experiments and generate the corresponding demonstrations and held-out probes. We first describe the STAGGER-style concept environment and the construction of individual target concepts in Sec.~\ref{app:stagger_des}. We then describe how demonstrations and held-out probes are sampled from each instantiated concept in Sec.~\ref{app:generating_demonstrations}.

\subsection{STAGGER-style Concept Environment}\label{app:stagger_des}
\vspace{-5pt}
We construct our controlled tasks using a STAGGER-style concept-learning environment~\citep{schlimmer1986incremental}, where each task is generated from a known concept with a finite and interpretable structure. This provides explicit ground truth for each concept-defining component.

\textbf{Concept representation.}
Following Sec.~\ref{sec:measurement}, we represent each target concept by its task-specific components:
\[
c_g=
\begin{cases}
(A,\{L_j\}_{j\in A}), & g\in\{\mathrm{AND},\mathrm{OR}\},\\
(A,\{L_j\}_{j\in A},m), & g=\mathrm{m\text{-}of\text{-}K},
\end{cases}
\]
where $g$ denotes the known composition form, $A$ is the effective relevant-attribute set, $L_j$ specifies the satisfying levels for attribute $j$, and $m$ is the composition threshold when applicable.

\textbf{Attribute information.} To provide a fixed set of distractor attributes, each input contains $M=K+4$ ordered attributes, where the nominal arity is $K\in\{3,4\}$ and thus $M\in\{7,8\}$. The generator selects $K$ attributes as predicate slots; the remaining four attributes are never selected as predicate slots. Each selected attribute is assigned one predicate, but some predicates may impose no effective constraint. We therefore define $A$ as the set of indices of selected attributes whose predicates affect the realized input--output mapping, so that $|A|\leq K$.

\textbf{Level constraints.}
Each attribute takes one of three ordered levels,
\[
\mathcal{V}=\{\textit{low},\textit{medium},\textit{high}\}.
\]
Identifying these levels with $\{0,1,2\}$ gives the input space $\mathcal{X}=\{0,1,2\}^{M}$. For each selected attribute, the generator samples an ordinal threshold predicate with an operator in $\{\leq,\geq\}$ and a threshold index in $\{0,1\}$. The resulting predicate space is
\[
\mathcal{Q}
=
\{(\leq,0),\,(\leq,1),\,(\geq,0),\,(\geq,1)\}.
\label{eq:clauses}
\]
These predicates correspond to the satisfying level sets
\[
(\leq,0)\mapsto\{\textit{low}\},\quad
(\leq,1)\mapsto\{\textit{low},\textit{medium}\},\quad
(\geq,0)\mapsto\mathcal{V},\quad
(\geq,1)\mapsto\{\textit{medium},\textit{high}\}.
\]
Thus, $A$ specifies \emph{which dimensions} participate in the realized concept, while $L_j$ specifies \emph{which levels} of each relevant dimension satisfy its predicate.

\textbf{Concept composition.} The $K$ predicates are combined using one of three known composition forms, $g\in\{\mathrm{AND},\mathrm{OR},\mathrm{m\text{-}of\text{-}K}\}$. Under \textsc{AND}, an input is positive if all predicates are satisfied; under \textsc{OR}, it is positive if at least one predicate is satisfied. Under $m$-of-$K$, an input is positive if at least $m$ predicates are satisfied, where $m=\left\lceil\frac{K+1}{2}\right\rceil$. Here, $m$ is the \emph{composition threshold}, which determines how many satisfied predicates are required for a positive output and is distinct from the threshold index used within an individual predicate.

\textbf{Non-effective predicates.} The predicate $(\geq,0)$ corresponds to $x_j\geq\textit{low}$ and is satisfied by every level in $\mathcal{V}$. It therefore imposes no effective constraint on the input. Under \textsc{OR}, such a predicate makes the concept always positive, while under $m$-of-$K$, it reduces the effective composition to an $(m-1)$-of-$(K-1)$ rule. We therefore reject concepts containing this predicate under \textsc{OR} and $m$-of-$K$ to preserve the intended nominal composition form and arity. Under \textsc{AND}, a non-effective conjunct does not change the resulting concept and is therefore allowed. Consequently, the nominal arity $K$ need not equal the effective number of relevant attributes $|A|$.

\textbf{Example.} Consider $K=3$, giving $M=7$ attributes. Suppose the generator selects $x_1,x_3,x_5$ and assigns 
\[x_1\geq\textit{medium},\qquad x_3\geq\textit{low},\qquad x_5\leq\textit{medium}.
\]
Under \textsc{AND}, the induced input--output mapping is
\[
c(x)= \mathbb{I}\!\left[
(x_1\geq\textit{medium})\land
(x_3\geq\textit{low})\land
(x_5\leq\textit{medium})\right].
\]
Here, the predicate on $x_3$ is non-effective and therefore does not affect the realized concept. The effective relevant-attribute set is consequently $A=\{1,5\}$, even though the nominal arity is $K=3$. The corresponding level constraints are
\[
L_1=\{\textit{medium},\textit{high}\},
\qquad
L_5=\{\textit{low},\textit{medium}\}.
\]

\textbf{Realized concepts.} Across $8$ generator seeds, $2$ nominal arities ($K\in\{3,4\}$), and $3$ composition forms, we instantiate $48$ target concepts in total, or $16$ concepts for each composition form. During concept recovery, the model is not given $K$; accordingly, the admissible recovery hypothesis class allows different cardinalities of the relevant-attribute set. We characterize this hypothesis class and its identifiability in Appendix~\ref{app:ident}.

\subsection{Generating Demonstrations and Held-out Probes}\label{app:generating_demonstrations}
\vspace{-5pt}
Given an instantiated concept specification $c_g$, we sample inputs and assign labels using the induced deterministic mapping $f_{c_g}$:
\[
y=f_{c_g}(x),\qquad x\in\mathcal{X}.
\]

\textbf{Demonstration generation.}
For each target concept, we construct demonstration sets with $H\in\{100,200\}$ examples. Each set is balanced between positive and negative examples, containing $H/2$ examples from each class, and the resulting examples are randomly shuffled to form
\[
D_H=\{(x^{(i)}, y^{(i)})\}_{i=1}^H.
\]
The same demonstration set is used as the evidence available to the model for both implicit prediction and concept recovery. Because the demonstrations are balanced, chance-level prediction accuracy is $0.50$.

\textbf{Held-out probes.} For each concept and evaluation condition, we construct a held-out query set $Q$ containing $32$ probes, with $16$ positive and $16$ negative examples. The probes are generated from the same underlying concept $c_g$ but are kept separate from the demonstrations. We sample probes without replacement from attribute configurations that do not occur in the demonstration set.

\textbf{Consistency across intervention conditions.} For comparisons across information-intervention conditions, we keep both the underlying concept and the demonstration set fixed, and vary only the information provided to the model before concept recovery. This ensures that differences in recovery can be attributed to the manipulated structural information rather than to differences in the sampled demonstrations.

\section{Demonstration Identifiability and Sufficiency}\label{app:ident}
\vspace{-5pt}
Before interpreting concept-recovery failures, we first determine whether the demonstrations uniquely identify the target concept. We formalize this question using the version space induced by the admissible hypothesis class. For a realized demonstration stream, we define the smallest prefix that leaves a single behaviorally distinct hypothesis as the \emph{concept-identification threshold}.

\subsection{Diagnostic Setup}\label{app:ident:setup}
\textbf{Admissible hypothesis class.} Following the concept construction in Appendix~\ref{app:stagger_des}, let $\mathcal{H}_{M,K',g}$ denote the set of valid concepts over $M$ attributes with nominal arity $K'$ and known composition form $g$. Because $g$ is specified to the model, we fix it to the target composition form. In contrast, the target nominal arity $K$ is not provided during concept recovery, so alternative hypotheses may have different nominal arities. We therefore define the admissible hypothesis class as
\[
\mathcal{H}_{\cup}=\bigcup_{K'=1}^{M}\mathcal{H}_{M,K',g}.
\]
This class contains all valid concepts that could be recovered under the information available to the model.

\textbf{Behavioral equivalence.} Different specifications in $\mathcal{H}_{\cup}$ may nevertheless induce the same input--output mapping. Since such specifications are indistinguishable from demonstrations, we identify hypotheses up to behavioral equivalence:
\[
c\equiv c'\quad\Longleftrightarrow\quad c(x)=c'(x)\ \ \forall x\in\mathcal{X}.
\label{eq:behavioral_equiv}
\]
Let $\overline{\mathcal{H}}_{\cup}$ denote the resulting set of behaviorally distinct equivalence classes.

\textbf{Version space.} Let $D_h=\{(x^{(i)},y^{(i)})\}_{i=1}^{h}$ denote the first $h$ demonstrations in a realized demonstration stream. The corresponding version space is the set of behaviorally distinct hypotheses consistent with all observed demonstrations:
\[
\overline{\mathcal{H}}_{\cup}(D_h)=\left\{[c]\in\overline{\mathcal{H}}_{\cup}\ :\
c(x^{(i)})=y^{(i)},\ \forall i\leq h\right\}.
\label{eq:version_space}
\]
A demonstration prefix is \emph{concept-identifying} if $\left|\overline{\mathcal{H}}_{\cup}(D_h)\right|=1$. 

\textbf{Stream--specific identification threshold.} For a target concept $c^\ast$ and a realized demonstration stream $D$, we define the concept-identification threshold as
\begin{equation}
H^{\ast}_{\mathrm{concept}}(c^{\ast};D)=\min\left\{h\ :\left |\overline{\mathcal{H}}_{\cup}(D_h)\right|=1\right\}.
\label{eq:hstar}
\end{equation}
This threshold measures how many demonstrations are required for this particular stream to uniquely identify the target concept within the admissible hypothesis class. It is therefore a property of the realized demonstration stream, rather than an intrinsic property of the target concept alone.

\subsection{General Sufficient-Budget Characterization}\label{app:ident:bound}
\vspace{-5pt}
The version-space formulation also provides a general sufficient condition for concept identification. Although the exact identification threshold can be computed in our finite controlled environment, the following characterization gives a model-independent upper bound for finite hypothesis classes. In particular, it relates the required demonstration budget to the number of behaviorally distinct alternatives and their separation from the target concept.

\textbf{Behavioral separation.} For a target concept $c^\ast$ and a behaviorally distinct competitor $c$, define the conditional disagreement rate for output $y\in\{0,1\}$ as
\[
\delta_y(c,c^\ast)
=
\Pr_{x\sim\mathcal{X}}
\left[
c(x)\neq c^\ast(x)
\;\middle|\;
c^\ast(x)=y
\right].
\label{eq:delta_y}
\]
We then define their balanced behavioral separation as
\[
\delta_{\mathrm{bal}}(c,c^\ast)=
\frac{\delta_0(c,c^\ast)+\delta_1(c,c^\ast)}{2}.
\label{eq:delta_bal}
\]
This quantity measures the average fraction of positive and negative contexts on which the competitor disagrees with the target. The balanced form matches the demonstration construction used in our experiments, which contains equal numbers of positive and negative examples.

Let
\[
\delta_{\min}(c^\ast)
=
\min_{\substack{[c]\in\overline{\mathcal H}_{\cup}\\
[c]\neq[c^\ast]}}
\delta_{\mathrm{bal}}(c,c^\ast)
\label{eq:delta_min}
\]
denote the minimum balanced behavioral separation between the target and any behaviorally distinct competitor.

\begin{proposition}[Sufficient demonstration budget]
Let $\overline{\mathcal H}_{\cup}$ be a finite set of behaviorally distinct hypotheses containing the target concept $c^\ast$. Suppose that demonstrations are sampled independently within each output class, with $H/2$ positive and $H/2$ negative examples, and that examples within each output class are sampled uniformly from the corresponding input subset. Then
\begin{equation}
\Pr\!\left[
\left|\overline{\mathcal H}_{\cup}(D_H)\right|>1
\right]
\leq
\left(
|\overline{\mathcal H}_{\cup}|-1
\right)
(1-\delta_{\min})^H.
\label{eq:union_bound}
\end{equation}
Consequently, for any $\eta\in(0,1)$, it is sufficient to use
\[
H_{\mathrm{suff}}(\eta)=
\left\lceil
\frac{
\log\!\left(|\overline{\mathcal H}_{\cup}|-1\right)
+
\log(1/\eta)
}{
-\log(1-\delta_{\min})
}
\right\rceil
\]
demonstrations to identify the target concept with probability at least $1-\eta$.
\end{proposition}

\begin{proof}
For a fixed competitor $c$, the probability that it agrees with the target on all $H/2$ positive demonstrations and all $H/2$ negative demonstrations is
\[
\left(1-\delta_1(c,c^\ast)\right)^{H/2}
\left(1-\delta_0(c,c^\ast)\right)^{H/2}.
\]
By the arithmetic–geometric mean inequality,
\[
\sqrt{
\left(1-\delta_1(c,c^\ast)\right)
\left(1-\delta_0(c,c^\ast)\right)
}
\leq
1-\delta_{\mathrm{bal}}(c,c^\ast),
\]
so the competitor-survival probability is at most
\[
(1-\delta_{\mathrm{bal}}(c,c^\ast))^H
\leq
(1-\delta_{\min})^H.
\]
Applying a union bound over the $|\overline{\mathcal H}_{\cup}|-1$ behaviorally distinct competitors gives Eq.~(\ref{eq:union_bound}). Requiring the resulting upper bound to be at most $\eta$ and solving for $H$ yields the stated sufficient budget.
\end{proof}
The bound is conservative because it protects against the simultaneous survival of every behaviorally distinct competitor, even though different competitors may be eliminated by overlapping subsets of demonstrations. It therefore provides a general sufficient budget rather than an estimate of the exact identification threshold for a realized demonstration stream.

\subsection{Exact Concept Identification}\label{app:ident:concept}
\vspace{-5pt}
In our controlled environment, the behaviorally distinct hypothesis space $\overline{\mathcal H}_{\cup}$ is finite and can be exhaustively enumerated. For each target concept and realized demonstration stream, we evaluate the version space after each demonstration prefix and compute the exact identification threshold $H^{\ast}_{\mathrm{concept}}$ from Eq.~(\ref{eq:hstar}). This provides a stream--specific measure of the number of demonstrations required to uniquely identify the target concept, without relying on the conservative sufficient-budget bound.

Table~\ref{tab:hstar} summarizes the resulting thresholds and the size of the remaining version space at several demonstration checkpoints. At $H=100$, $47$ of the $48$ realized concept--stream instances are uniquely identified, corresponding to an identification rate of $97.9\%$; at $H=200$, all $48$ realized concept--stream instances are identified. The single unresolved stream at $H=100$ is an \textsc{OR} concept with $K=4$, for which exactly one behaviorally distinct competitor remains. Under the fixed-arity analysis, all $96$ concept--budget pairs are identified.

The exact enumeration confirms that the deployed demonstration budgets are sufficient for concept identification in all but one realized stream at $H=100$, and for all realized streams at $H=200$. At the same time, the variation in identification thresholds across composition forms and demonstration streams indicates that the amount of evidence required for identification is not uniform across tasks.

\begin{table}[ht]
\centering
\vspace{-10pt}
\caption{Exact concept-identification thresholds and version-space contraction over $\overline{\mathcal H}_{\cup}$. The single stream that remains unidentified at $H=100$ (\textsc{OR}, $K=4$) is excluded from the threshold statistics. For each form--arity pair, \textit{med.} and \textit{p90} report the ranges of the corresponding statistics across the independently realized demonstration streams under $H=100$ and $H=200$, while \textit{max} reports the largest observed identification threshold among the identified streams. The last three columns report the median number of additional behaviorally distinct rival hypotheses remaining after $5$, $12$, and $30$ demonstrations.}
\vspace{2pt}
\label{tab:hstar}
\begin{tabular}{llrrrrrr}
\toprule
form & $K$ & med. & p90 & max & $h=5$ & $h=12$ & $h=30$ \\
\midrule
\textsc{and}    & 3 & 14--16 & 24--29 & 43 & $1{,}013$  & 24  & 0 \\
\textsc{and}    & 4 & 15--17 & 26--30 & 31 & $2{,}091$  & 88  & 0 \\
\textsc{m-of-K} & 3 & 25     & 29--33 & 37 & $9{,}881$  & 251 & 0 \\
\textsc{m-of-K} & 4 & 26--32 & 39--46 & 56 & $41{,}537$ & 660 & 4 \\
\textsc{or}     & 3 & 16--24 & 26--35 & 48 & $82$       & 2   & 0 \\
\textsc{or}     & 4 & 31--52 & 60--72 & 76 & $95$       & 3   & 1 \\
\bottomrule
\end{tabular}
\vspace{-10pt}
\end{table}

\subsection{Query-Level Demonstration Sufficiency}\label{app:ident:query}
\vspace{-5pt}
Full concept identification is stronger than what is required to determine the label of an individual query. A version space may contain multiple behaviorally distinct concepts while all remaining concepts agree on the label of a particular input. We therefore define the number of labels supported by the version space for a query $x_q$ as
\[
    N_q(D_h,x_q)=\left|\left\{c(x_q):[c]\in\overline{\mathcal{H}}_{\cup}(D_h)\right\}\right|.
\]
The query is identified when $N_q(D_h,x_q)=1$, with exact query-level threshold
\[
H^{\ast}_{\mathrm{query}}(x_q;D)=\min\left\{h:N_q(D_h,x_q)=1\right\}.
\]
We compute this quantity for the same held-out probes used in ICL evaluation. Each concept--stream instance contains $32$ probes, yielding $1{,}536$ probes at each demonstration budget. At $H=100$, $99.9\%$ of probes are uniquely identified, with only one unresolved probe; at $H=200$, all probes are identified. Across the twelve form--arity--budget cells, the median $H^{\ast}_{\mathrm{query}}$ ranges from $7$ to $21$, compared with $14$–$52$ across the corresponding form–arity cells for concept-level identification.

Restricting the analysis to the $95$ concept--stream instances that are identified within the corresponding demonstration budget, $82.4\%$ of evaluation probes become identifiable strictly before the full concept is identified, while the remaining $17.6\%$ become identifiable only when the final behaviorally distinct rival is eliminated. Thus, query-level sufficiency is typically reached earlier than full concept identification, although some queries remain tied to the final concept-level ambiguity.

Over the full input domain, concept-level identification is equivalent to identifying the labels of every possible input:
\[
H^{\ast}_{\mathrm{concept}}=\sup_{x\in\mathcal{X}}H^{\ast}_{\mathrm{query}}(x).
\]
The equality follows because a version space is a singleton if and only if all hypotheses in it agree on every $x\in\mathcal{X}$, while any pair of behaviorally distinct hypotheses disagrees on at least one such input.

\subsection{Information Structure and Sampling Effects}\label{app:ident:info}
\vspace{-5pt}
The exact identification threshold varies across realized demonstration streams even for the same target concept. To distinguish intrinsic structure from sampling effects, we generate additional demonstration streams for each target concept and decompose the variation in $H^{\ast}_{\mathrm{concept}}$. Across these streams, between-concept variation accounts for only $4$--$28\%$ of the variance, with the remaining $72$--$96\%$ corresponding to variation across sampled demonstration streams. Thus, identification difficulty is primarily a property of the target--stream pair rather than an intrinsic property of the concept alone.

The information carried by positive and negative demonstrations also depends on the composition form. For each behaviorally distinct competitor, we record the label of the first demonstration that eliminates it from the version space. For \textsc{AND} concepts, positive examples eliminate $87.5$–$90.1\%$ of competitors first, whereas for \textsc{OR} concepts, negative examples eliminate $99.3$–$99.8\%$ first; \textsc{m-of-K} lies between these cases. This asymmetry reflects the different logical constraints imposed by the three composition forms and helps explain variation in their identification thresholds.

Finally, we verify that balanced demonstration sampling does not itself create identifiability. Replacing the balanced sampler with natural sampling from the uniform input distribution changes the median $H^{\ast}_{\mathrm{concept}}$ by at most four demonstrations and leaves identification rates unchanged at the deployed budgets. Balanced sampling therefore controls demonstration label frequency without materially altering the evidence required to identify the target concept.

These diagnostics establish the information-side conditions for interpreting the subsequent concept-recovery experiments: the deployed demonstrations uniquely identify the target concept for nearly all realized streams at $H=100$ and for all evaluated streams at $H=200$, while the amount and type of evidence required for identification vary systematically across streams and composition forms.

\section{Additional Results and Experimental Details}
\vspace{-5pt}

\subsection{Prompt Templates}\label{app:prompt}

We use the STAGGER-style synthetic environment described in Sec.~\ref{sec:measurement}. We use a modular prompt design with three base templates: \emph{implicit prediction}, \emph{structured concept recovery}, and \emph{concept application}. Unless otherwise noted, the shared components of each template remain unchanged across experimental conditions. The templates below use an \textsc{AND} instance for illustration; the composition-specific instruction is substituted for \textsc{OR} and $m$-of-$K$ conditions. All examples use the same 200 demonstrations, truncated for display. Here HURT denotes $y=0$ and HELP denotes $y=1$.

\textbf{Implicit Prediction.} The implicit-prediction template evaluates whether the model can infer the input--output relationship directly from demonstrations and predict the label of a held-out query. No concept representation is requested from the model.

\begin{promptbox}{Implicit Prediction}
\footnotesize
\textbf{SYSTEM}\vspace{2pt}\\
You predict whether an action HELPS or HURTS for an item, using ONLY the observed examples. Answer with exactly one word: HELP or HURT. \\
 
\vspace{4pt}\hrule\vspace{4pt}
\textbf{USER}\vspace{2pt}\\
Observed examples:\\
\texttt{\ \ x1=high x2=low x3=medium x4=high x5=medium x6=low x7=low -> HURT}\\
\texttt{\ \ x1=high x2=high x3=high x4=high x5=medium x6=high x7=low -> HURT}\\
\hfill $\vdots$ \hfill {\itshape (200 demonstrations total)} \hfill\hfill\\
\texttt{\ \ x1=high x2=low x3=high x4=low x5=high x6=low x7=high -> HURT}
 
\vspace{3pt}
Query: \texttt{x1=medium x2=low x3=high x4=medium x5=low x6=high x7=high} \\
Answer (HELP or HURT):
\end{promptbox}

\vspace{5pt}
\textbf{Structured Concept Recovery.} The structured-recovery template asks the model to reconstruct a concept from the demonstrations. The template contains two independently varied components: \emph{concept information}, which provides task-defining structural information when available, and \emph{recovery procedure}, which determines how the model reconstructs the concept.

\begin{promptbox}{Structured Concept Recovery}
\footnotesize
\textbf{SYSTEM}\vspace{2pt}\\
You maintain a RULE predicting HELP vs HURT from an item's features. The rule is a CONJUNCTION: an item HELPS if and only if EVERY listed feature has one of its allowed levels; features you do not list are irrelevant. Levels are only: low, medium, high.\\
Output the rule as lines exactly of the form \texttt{feature: level, level} (allowed levels for HELP). List ONLY features that matter. Output ONLY these lines and nothing else; if no feature matters yet, output the single word NONE. Change the rule as little as possible.

\textbf{Concept Information.} The concept-information slot controls which task-defining components are provided to the model in addition to the demonstrations. The baseline condition provides no additional concept information. The remaining conditions reveal selected components of the target concept.
 
\vspace{4pt}\hrule\vspace{4pt}
\textbf{USER}\vspace{2pt}\\
CURRENT RULE:\\
NONE
 
\vspace{3pt}
NEW OBSERVATIONS:\\
\texttt{\ \ x1=high x2=low x3=medium x4=high x5=medium x6=low x7=low -> HURT}\\
\hfill $\vdots$ \hfill {\itshape (200 demonstrations total)} \hfill\hfill\\
\texttt{\ \ x1=medium x2=low x3=high x4=medium x5=low x6=high x7=high -> HELP}
 
\vspace{5pt}
\slot{Concept Information (optional; $\varnothing$ by default --- nothing is inserted here unless the condition specifies otherwise)}
 
\vspace{5pt}
\slot{Recovery Procedure (Direct by default --- the line below is used as-is; CoT/Iterative replace or wrap it, see below)}\\
UPDATED RULE (feature: allowed levels for HELP):
 
\end{promptbox}

\begin{modbox}{Concept Information}
\begin{description}[leftmargin=1.3cm,style=nextline,itemsep=2pt]
\item[$\varnothing$] Nothing is inserted; this is the baseline recovery condition.
\item[$K$] The prompt specifies that exactly $K$ features are relevant and asks the model to identify which $K$ are relevant. The available features are listed as \texttt{x1, x2, \ldots, x7}.
\item[$A$] The prompt directly provides the relevant features: \texttt{RELEVANT FEATURES (use ONLY these): x1, x2, x4}.
\item[$L$] The prompt provides the level constraints associated with the target concept, without directly providing the relevant feature identities.
\item[$(A,L)$] The prompt directly provides both the relevant features and their corresponding level constraints: \texttt{x1: low, medium; x2: low, medium; x4: low, medium}.
\item[$(A,L,m)$] For $m$-of-$K$ concepts, the prompt additionally provides the composition threshold $m$.
\item[$(A,L,\tilde L)$] The prompt provides $(A,L)$ together with one additional plausible but irrelevant level constraint $\tilde L$, which does not belong to the target concept.
\end{description}
\end{modbox}

\begin{modbox}{Recovery Procedure}
\textbf{Direct} (default): no change --- the base template ends directly at ``UPDATED RULE (feature: allowed levels for HELP):''.
 
\vspace{3pt}
\textbf{CoT}: the SYSTEM prompt appends \textit{``Reason BRIEFLY (a few sentences) about the observations, THEN on a new line write \texttt{RULE:} and, on the lines AFTER it, ONE feature per line... Put NO prose, reasons, or other words inside the RULE block.''} plus a one-shot example; the USER prompt's closing line becomes \textit{``Reason briefly, then end with the RULE: block:''} in place of ``UPDATED RULE:''. Only the text after \texttt{RULE:} is parsed and scored against $c^\ast$.
 
\vspace{3pt}
\textbf{Iterative} (3 turns; Direct is reused verbatim as Turn 1): Turn 1 = the full Direct template above $\to \hat c_1$. Turn 2 (critique, new call): SYSTEM \textit{``You are given a candidate RULE and labeled examples. Check the rule against EVERY example; list which examples it misclassifies and what is inconsistent... Do NOT output a rule --- only the critique.''}, USER supplies $\hat c_1$ + the same 200 demonstrations. Turn 3 (correct, new call): SYSTEM \textit{``You are given a candidate RULE, a critique of it, and the examples. Produce the CORRECTED final rule... Output ONLY lines \texttt{feature: level, level}... or NONE.''}, USER supplies $\hat c_1$ + the Turn-2 critique + demonstrations. Only the Turn-3 output is scored.
\end{modbox}

\vspace{4pt}
\textbf{Concept Application.} The concept-application template evaluates whether the model can apply a supplied concept to a held-out query. We vary the representation of the supplied concept and, where applicable, the procedure used to apply it.

\begin{promptbox}{Concept Application}
\footnotesize
\textbf{SYSTEM}\vspace{2pt}\\
You are given a RULE (lines \texttt{feature: allowed levels}). HELP iff EVERY listed feature has one of its allowed levels. Apply it exactly as written; do not second-guess it. Answer with exactly one word: HELP or HURT.
 
\vspace{4pt}\hrule\vspace{4pt}
\textbf{USER}\vspace{2pt}\\
RULE:\\
\slot{Concept Representation --- see variants below}
 
\vspace{3pt}
Item: \texttt{x1=medium x2=low x3=high x4=medium x5=low x6=high x7=high} \\
Answer (HELP or HURT):\hfill
\end{promptbox}

\begin{modbox}{Concept Representation}
\begin{description}[leftmargin=1.6cm,style=nextline,itemsep=2pt]
\item[DSL] \texttt{x1: low, medium}\ \ \texttt{x2: low, medium}\ \ \texttt{x4: low, medium} \quad(shown in Template 3 above)
\item[Inequality] \texttt{x1<=medium AND x2<=medium AND x4<=medium}\ \ (system swaps to: \textit{``\texttt{fi>=L} means attribute $i$ is at level $L$ or higher; \texttt{fi<=L} means $L$ or lower.''})
\item[Natural language] \textit{``An item HELPS (else HURTS) if all of the following hold: $x_1$ is low or medium; $x_2$ is low or medium; and $x_4$ is low or medium.''}
\end{description}
\end{modbox}

\subsection{Level-Constraint Error Characterization}\label{app:level_error}
We further characterize the level-constraint errors underlying the $\mathrm{Acc}_{L|A}$ metric in Table~\ref{result:implicit_vs_explicit}. Let $A$ and $\hat A$ denote the ground-truth and recovered relevant-attribute sets, and let $L_j$ and $\hat L_j$ denote the ground-truth and recovered level constraints for attribute $j$. We consider only attributes that are correctly identified as relevant, i.e., $j\in A\cap\hat A$. Exact matches $\hat L_j=L_j$ are counted as correct in the main-table metric; otherwise, the mismatch is categorized by the set relation between $\hat L_j$ and $L_j$.

Specifically, \emph{Over} denotes $\hat L_j\supset L_j$, \emph{Under} denotes $\hat L_j\subset L_j$, \emph{Partial} denotes overlap without containment, i.e., $\hat L_j\cap L_j\neq\varnothing$ with neither set containing the other, and \emph{Disjoint} denotes $\hat L_j\cap L_j=\varnothing$. For example, ${\textit{low}}\rightarrow\{\textit{low},\textit{medium}\}$ is \emph{Over}, whereas $\{\textit{low},\textit{medium}\}\rightarrow\{\textit{medium},\textit{high}\}$ is \emph{Partial}.

Table~\ref{tab:level_error} reports the distribution of mismatches. \emph{Over} is the dominant error type for most model--form pairs, particularly under \textsc{OR}, while \emph{Under} is comparatively rare. Thus, incorrect level recovery is generally biased toward including additional satisfying levels.

\begin{table}[ht]
\centering
\small
\caption{Distribution of level-constraint mismatches by error type. Values are computed among attributes correctly identified as relevant but with mismatched level constraints. \emph{Over}, \emph{Under}, \emph{Partial}, and \emph{Disjoint} denote the corresponding set relations between recovered and ground-truth level constraints.}
\vspace{2pt}
\label{tab:level_error}
\begin{tabular}{llcccc}
\toprule
Form & Model
& Over
& Under
& Partial
& Disjoint \\
\midrule
\multirow{4}{*}{\textsc{AND}}
& Qwen2.5-72B   & 0.65 & 0.25 & 0.10 & 0.00 \\
& Llama-3.1-70B & 0.58 & 0.11 & 0.32 & 0.00 \\
& Gemma-2-27B   & 0.76 &  0.03 & 0.21 & 0.00 \\
& Mistral-24B   &  0.09 & 0.32 & 0.59 & 0.00 \\
\midrule
\multirow{4}{*}{\textsc{OR}}
& Qwen2.5-72B   & 0.96 &  0.01 & 0.00 & 0.03 \\
& Llama-3.1-70B & 0.93 &  0.00 & 0.01 & 0.06 \\
& Gemma-2-27B   & 0.72 &  0.00 & 0.01 & 0.27 \\
& Mistral-24B   & 0.63 &  0.02 & 0.01 & 0.33 \\
\midrule
\multirow{4}{*}{$m$-of-$K$}
& Qwen2.5-72B   & 0.59 &  0.03 & 0.27 & 0.11 \\
& Llama-3.1-70B & 0.59 &  0.04 & 0.17 & 0.21 \\
& Gemma-2-27B   & 0.48 &  0.05 & 0.24 & 0.24 \\
& Mistral-24B   & 0.32 &  0.06 & 0.21 & 0.40 \\
\bottomrule
\end{tabular}
\end{table}

\subsection{Threshold Error Characterization}\label{app:threshold_error}
\vspace{-5pt}
We further characterize the threshold errors underlying the \emph{$\mathrm{Acc}_m$} metric in Table~\ref{result:implicit_vs_explicit}. Let $m$ and $\hat m$ denote the ground-truth and recovered thresholds, respectively. We consider only $m$-of-$K$ concepts with incorrectly recovered thresholds, i.e., $\hat m\neq m$. We classify each error by direction: \emph{Too High} denotes $\hat m>m$, while \emph{Too Low} denotes $\hat m<m$. Table~\ref{tab:threshold_error} reports the resulting error distribution. Across most models, threshold errors are biased toward overestimating the number of satisfied predicates required for a positive prediction. Qwen, Gemma, and Mistral make almost exclusively \emph{Too High} errors, whereas Llama shows a more balanced pattern.

\begin{table}[ht]
\centering
\small
\caption{Distribution of threshold errors by error direction. Values are computed among $m$-of-$K$ concepts with incorrectly recovered composition thresholds. \emph{Too High} denotes $\hat m>m$, while \emph{Too Low} denotes $\hat m<m$.}
\vspace{2pt}
\label{tab:threshold_error}
\begin{tabular}{llcc}
\toprule
Form & Model
& Too High
& Too Low \\
\midrule
\multirow{4}{*}{$m$-of-$K$}
& Qwen2.5-72B   & 1.00 & 0.00 \\
& Llama-3.1-70B &  0.50 & 0.50 \\
& Gemma-2-27B   &  1.00 & 0.00\\
& Mistral-24B   &  1.00 & 0.00 \\
\bottomrule
\end{tabular}
\end{table}

\subsection{Concept Application Control}\label{app:addtion_apply}
\vspace{-5pt}
We evaluate model performance when the target concept is provided directly, without requiring concept recovery. Table~\ref{tab:apply} reports held-out query accuracy across the three concept families and multiple concept representations. Models generally achieve substantially higher accuracy when the concept is directly provided than when it must be recovered from demonstrations. The deterministic executor achieves $1.00$ throughout, confirming that the provided concepts are executable under the evaluation procedure.

This control suggests that the prediction–recovery gap is not simply attributable to difficulty in applying an explicitly specified concept. Rather, models can generally apply a provided concept with high accuracy, while recovering the underlying concept from demonstrations remains substantially more difficult.

\vspace{-10pt}
\begin{table}[ht]
\caption{Held-out query accuracy on $Q$ when the target concept $c$ is directly provided in different representations. \emph{DSL}, \emph{Inequality}, and \emph{Natural Language} express the same concept in symbolic, inequality-based, and natural-language forms, respectively. \emph{DSL + CoT} additionally provides a reasoning trace, while \emph{Executor} applies the concept using a deterministic symbolic executor.}
\vspace{2pt}
\label{tab:apply}
\centering
\small
\begin{tabular}{llccccc}
\toprule
Form & Model & DSL & Inequality & Natural & DSL + CoT & Executor \\
\midrule
\multirow{4}{*}{\textsc{AND}}
 & Qwen2.5-72B   & 0.94 & 0.97 & 0.88 & 0.98 & 1.00 \\
  & Llama-3.1-70B & 0.99 & 0.96 & 1.00 & 0.98 & 1.00 \\
 & Gemma-2-27B   & 0.90 & 0.82 & 0.87 & 0.98 & 1.00 \\
 & Mistral-24B   & 0.80 & 0.81 & 0.91 & 1.00 & 1.00 \\
\midrule
\multirow{4}{*}{\textsc{OR}}
 & Qwen2.5-72B   & 0.97 & 0.99 & 0.99 & 0.99 & 1.00 \\
& Llama-3.1-70B & 0.76 & 0.95 & 0.94 & 1.00 & 1.00 \\
 & Gemma-2-27B   & 0.94 & 0.87 & 0.92 & 0.81 & 1.00 \\
 & Mistral-24B   & 0.64 & 0.97 & 0.93 & 1.00 & 1.00 \\
\midrule
\multirow{4}{*}{$m$-of-$K$}
 & Qwen2.5-72B   & 0.79 & 0.87 & 0.81 & 0.98 & 1.00 \\
  & Llama-3.1-70B & 0.74 & 0.77 & 0.80 & 0.98 & 1.00 \\
 & Gemma-2-27B   & 0.83 & 0.83 & 0.78 & 0.86 & 1.00 \\
 & Mistral-24B   & 0.68 & 0.73 & 0.76 & 0.99 & 1.00 \\
\bottomrule
\end{tabular}
\vspace{-5pt}
\end{table}

\subsection{Additional Benchmark Datasets}\label{app:other_dataset}
\vspace{-5pt}
We further examine whether the component-level over-inclusion pattern extends beyond the controlled STAGGER-style environment. We consider two complementary settings: \textsc{MONK-1}, which has a specified target concept and relevant and irrelevant attributes~\footnote{\url{https://archive.ics.uci.edu/dataset/70/monk}}, and Mushroom, a real-world categorical classification dataset without a dataset-official ground-truth concept~\footnote{\url{https://archive.ics.uci.edu/dataset/73/mushroom}}. MONK-1 therefore permits both behavioral and feature-level recovery evaluation, whereas Mushroom provides an external setting for examining attribute selection without assuming an official ground-truth concept.

As shown in Table~\ref{tab:datasets}, recovered concepts tend to include more attributes than the available or documented relevant-attribute reference on both datasets, except for Llama on MONK-1. This pattern is consistent with the over-inclusion observed in the controlled setting.

\begin{table}[ht]
\vspace{-10pt}
\caption{Concept application and recovery on additional benchmark datasets. $Q_{\mathrm{imp}}$ denotes direct prediction from the demonstrations, while $Q_{\mathrm{exp}}$ denotes prediction obtained by executing the recovered specification on held-out queries. Since the query sets are balanced, $0.50$ corresponds to random guessing. The remaining columns report attribute-level recovery. $A$-Precision and $A$-Recall measure precision and recall of the recovered relevant-attribute set on MONK-1; Ref. $|A|$ denotes the documented reference feature count and $|\hat A|$ the average number of recovered attributes.}
\vspace{2pt}
\label{tab:datasets}
\centering
\small
\begin{tabular}{llcc|cccc}
\toprule
& & \multicolumn{2}{c|}{Concept Application}
& \multicolumn{4}{c}{Concept Recovery} \\
  Dataset   & Model
& $Q_{\mathrm{imp}}$
& $Q_{\mathrm{exp}}$
& $A$-Precision
& $A$-Recall
& Ref. $|A|$
& $|\hat A|$ \\
\midrule
\multirow{4}{*}{\textsc{MONK-1}}
& Qwen2.5-72B   & 0.81 &  0.52 & 0.43 & 0.67 & 3 & 4.6 \\
& Llama-3.1-70B & 0.81 & 0.49 & 0.70 & 0.40 & 3 & 1.8 \\
& Gemma-2-27B   & 0.65 &  0.52 & 0.39 & 0.67 & 3 & 5.0 \\
& Mistral-24B   & 0.74 & 0.48 & 0.59 & 0.73 & 3 & 3.8 \\
\midrule
\multirow{4}{*}{\textsc{Mushroom}}
& Qwen2.5-72B   & 0.96 & 0.67  & – & – & 4 & 7.8 \\
& Llama-3.1-70B & 0.73 & 0.56  & – & – & 4 & 4.4 \\
& Gemma-2-27B   & 0.76 & 0.46  & – & – & 4 & 12.0 \\
& Mistral-24B   & 0.78 & 0.50  & – & – & 4 & 13.4 \\
\bottomrule
\end{tabular}
\vspace{-5pt}
\end{table}

\subsection{Inference-Time Reasoning and Self-Refinement}\label{app:cot_selfrefine}

We further examine whether inference-time reasoning can improve recovered-concept prediction when the available structural information is incomplete. In particular, we consider two commonly used inference-time interventions: chain-of-thought reasoning~\citep{wei2022chain} and iterative self-refinement~\citep{madaan2023self}. These interventions leave the demonstrations and recovery procedure unchanged and modify only the inference process used to produce the recovered concept.

We evaluate the resulting recovered concepts by applying them to the same held-out query set $Q$ and reporting $Q_{\mathrm{exp}}$. We consider two information settings. In the first, the model receives demonstrations alone and must recover the concept without additional structural information. In the second, we provide the identities of the ground-truth relevant attributes $A$, while withholding their level constraints and the composition rule. This setting tests whether inference-time reasoning can make effective use of partial structural information.

For chain-of-thought reasoning, the model is prompted to reason about the demonstrations and explicitly derive the concept before producing the recovered specification. For iterative self-refinement, the model first produces a recovered concept, then revisits its prediction and revises the specification based on its own intermediate result. We use the same evaluation protocol and held-out queries across all inference conditions.

Table~\ref{tab:cot_selfrefine} summarizes the results. Neither chain-of-thought reasoning nor iterative self-refinement consistently improves recovered-concept prediction over direct recovery. This remains the case when the identities of the relevant attributes are provided: although some model--form combinations improve, others remain unchanged or degrade. Thus, the tested inference-time interventions do not reliably compensate for incomplete structural information.

Overall, these results provide a complementary sanity check on the structural-information analysis in Sec.~\ref{exp:localize}. The recovery gains observed when complementary structural components are jointly specified are not reproduced simply by adding inference-time reasoning or self-refinement.

\begin{table}[ht]
\vspace{-10pt}
\caption{Inference-time interventions for recovered-concept prediction. $Q_{\mathrm{exp}}$ denotes prediction accuracy on held-out queries after executing the recovered concept; $Q_{\mathrm{exp}}(A)$ additionally provides the identities of the ground-truth relevant attributes, without their level constraints or composition rule. Superscripts $\mathrm{CoT}$ and $\mathrm{Iter}$ denote chain-of-thought reasoning and iterative self-refinement, respectively.}
\vspace{2pt}
\label{tab:cot_selfrefine}
\centering
\small
\begin{tabular}{llccccccc}
\toprule
Form & Model & $Q_{\mathrm{exp}}$ & $Q_{\mathrm{exp}}^{\mathrm{CoT}}$ & $Q_{\mathrm{exp}}^{\mathrm{Iter}}$ & $Q_{\mathrm{exp}}(A)$ & $Q_{\mathrm{exp}}^{\mathrm{CoT}}(A)$ & $Q_{\mathrm{exp}}^{\mathrm{Iter}}(A)$\\
\midrule
\multirow{4}{*}{\textsc{AND}}
 & Qwen2.5-72B    & 0.63 & 0.54 & 0.56 & 0.95 & 0.81& 0.83 \\
 & Llama-3.1-70B  & 0.56 & 0.59 & 0.55 & 0.96 & 0.83& 0.94 \\
 & Gemma-2-27B    & 0.58 & 0.52 & 0.54 & 0.64 & 0.61& 0.65 \\
 & Mistral-24B    & 0.52 & 0.52 & 0.51 & 0.83 & 0.74& 0.68 \\
\midrule
\multirow{4}{*}{\textsc{OR}}
 & Qwen2.5-72B    & 0.50 & 0.53 & 0.51 & 0.52 & 0.65& 0.53 \\
 & Llama-3.1-70B  & 0.51 & 0.56 & 0.52 & 0.54 & 0.66 & 0.56\\
 & Gemma-2-27B    & 0.50 & 0.48 & 0.50 &  0.49& 0.43 & 0.49\\
 & Mistral-24B    & 0.51 & 0.56 & 0.51 & 0.52 & 0.57 & 0.56\\
\midrule
\multirow{4}{*}{$m$-of-$K$}
 & Qwen2.5-72B    & 0.53 & 0.53 & 0.54 & 0.54 & 0.50& 0.53 \\
 & Llama-3.1-70B  & 0.51 & 0.52 & 0.51 & 0.50 & 0.50 & 0.52\\
 & Gemma-2-27B    & 0.51 & 0.49 & 0.51 & 0.47 & 0.44 & 0.50\\
 & Mistral-24B    & 0.51 & 0.47 & 0.51 & 0.47 & 0.41 & 0.47\\
\bottomrule
\end{tabular}
\vspace{-10pt}
\end{table}

\subsection{Training Details}\label{app:training_details}

This appendix describes the supervised training procedure used for the results in Sec.~\ref{exp:repair}. The supervised models are trained on $m$-of-$K$ concepts; \textsc{AND} and \textsc{OR} concepts are used only for cross-family transfer evaluation.

\textbf{Training Data.} We generate training concepts using the same procedure as in the main experiments, with $K\in\{3,4\}$ and demonstration budgets $H\in\{100,200\}$. Each training concept specifies its relevant attributes, level constraints, and composition threshold $m$. We construct $3,000$ training concepts. For \emph{recovery-only} training, each concept contributes one structured-recovery example mapping demonstrations to the target concept, yielding $3,000$ training examples. For \emph{joint} training, each concept additionally contributes two balanced implicit-prediction examples, one for each label, yielding $9,000$ examples in total. Recovery and prediction examples are interleaved and jointly optimized.

\textbf{Evaluation Split.} To reduce the possibility of memorizing specific relevant-attribute combinations, held-out evaluation concepts use relevant-attribute combinations that are disjoint from those used for training. The evaluation set contains $570$ concepts, including $270$ with $K=3$ and $300$ with $K=4$.

\textbf{Training Objective.} Both structured recovery and implicit prediction are trained with the standard autoregressive next-token cross-entropy objective. We use completion-only masking: prompt tokens are excluded from the loss, which is computed only over the target completion. Recovery examples target the structured concept specification, whereas implicit-prediction examples target the corresponding binary label.

\textbf{Model and Optimization.} We fine-tune Qwen2.5-7B-Instruct with LoRA while keeping the backbone parameters frozen. LoRA uses rank $r=16$, scaling factor $\alpha=32$, and dropout $0.05$, introducing approximately $40$M trainable parameters. We use AdamW with a learning rate of $1\times10^{-4}$, cosine learning-rate decay, and a $3\%$ warmup ratio. Training uses bf16 precision.

\end{document}